\documentclass[final,12pt]{colt2020-plain}

\title[Characterization of Adversarial Bandit Learnability]{A Complete Characterization of Learnability\\ for Adversarial Noisy Bandits}
\usepackage{times}
\usepackage{mathrsfs}

\altauthor{%
 \Name{Steve Hanneke} \Email{steve.hanneke@gmail.com}\\
 \addr Purdue University
 \AND
 \Name{Kun Wang} \Email{wangkun8512@gmail.com}\\
 \addr Purdue University
}

\begin{document}

\maketitle

\begin{abstract}%
We study adversarial noisy bandits given a known function class $\mathcal{F}$. In each round, the adversary selects a function $f \in \mathcal{F}$, the learner chooses an arm, and then observes a noisy reward determined by the chosen arm and the function $f$. The goal is to minimize the cumulative regret $R(T)$, defined as the difference between the learner's performance and that of the best fixed arm in hindsight over $T$ rounds. We say that a function class $\mathcal{F}$ is learnable if there exists an algorithm achieving sublinear regret.  Our main result is a complete characterization of learnability for adversarial noisy bandits. The characterization is given in terms of a convexified variant of the \emph{generalized maximin volume} introduced by \citet{hanneke2025complete}: namely, the generalized maximin volume evaluated on the convex hull $\operatorname{co}(\mathcal F)$. We prove that $\mathcal F$ is learnable if and only if this convexified generalized maximin volume is positive at every scale. This condition characterizes learnability against both oblivious and adaptive adversaries, showing in particular that these two notions of learnability are equivalent in the noisy bandit setting. Our analysis reveals that the key complexity measure is closely connected to two new combinatorial notions, \emph{hitting set} and  \emph{distribution covering number}, which may be of independent interest. These results establish the first complete characterization of learnability for adversarial noisy bandits.
\end{abstract}

\begin{keywords}%
Learning Theory, Bandits, Adversarial Bandits, Learnability
\end{keywords}

\section{Introduction}
\label{intro}

The multi-armed bandit problem \citep{robbins1952some,auer2002finite,auer2002nonstochastic,bubeck2012regret,slivkins2019introduction,lattimore2020bandit} is a classic sequential decision making problem. In each round, the learner chooses an action (also known as an arm) and observes a reward associated with that arm. The goal of the learner is to choose between multiple arms over time to maximize the cumulative rewards. Two fundamental models have been studied in parallel: the stochastic bandit model and the adversarial bandit model. In this work, we focus on the adversarial bandit setting. 

Compared to the stochastic bandit, the adversarial bandit \citep{auer1995gambling,auer2002nonstochastic} allows rewards to be chosen arbitrarily by an adversary, rather than being sampled from fixed underlying distributions. As a result, it provides a more robust model for non-stationary and potentially malicious environments. For example, consider a customer who repeatedly buys apples from different companies over a long period of time. The quality of the apples from each company may vary due to many factors that change over time, such as seasonality, transportation conditions, and storage quality. In such a situation, a strategy designed to be robust to adversarial or non-stationary environments can gradually achieve better performance over time. Other motivating examples include dynamic pricing, online advertising, and recommendation systems. As noted by \citet{auer2002nonstochastic}, adversarial bandits also have broad applications in playing repeated games.

What we explore in this work is the role of \emph{function class} in this problem. Since the introduction of the adversarial bandit model, there has been a substantial body of work on this problem. However, most existing work either assumes no relationship among the rewards of different arms \citep{auer1995gambling,auer2002nonstochastic,bubeck2012regret,neu2015explore,putta2022scale}, or imposes some particular structural assumption on the reward function, such as linearity \citep{bubeck2014entropic,JMLR:v17:hazan16a,hoeven2018many,neu2020efficient} or Lipschitz continuity \citep{maillard2010online,podimata2021adaptive}. An abstract formulation of the adversarial bandit problem with a function class can be viewed as a special case of the adversarial decision making problem introduced by \citet{foster2022complexity}: there is an arm set $\Pi$. A function $f$ maps arms $\pi \in \Pi$ to the mean reward $f(\pi) \in [0,1]$ of the underlying distributions. A function class is a collection of measurable functions. For any function class $\mathcal{F}$, we consider any distributions supported on $[0,1]$ with a mean reward function $f \in \mathcal{F}$. In other words, we consider \emph{arbitrary bounded noise} in this work.\footnote{
Similar to \cite{hanneke2025complete}, the noise distribution could be extended to binary noise and Gaussian noise, etc.} The learning problem induced by the function class $\mathcal{F}$ is as follows: the game proceeds over $T$ rounds. At each round $t$, the adversary chooses a function $f_t$ from the function class $\mathcal{F}.$ The learner then selects an arm $\pi_t \in\Pi$, and receives a noisy reward $r_t$ whose conditional mean given the past history is $f_t(\pi_t)$. The objective is to minimize regret $\sup_{\pi^*}\sum_{t=1}^T (f_t(\pi^*)-f_t(\pi_t))$, which measures the gap between the learner’s cumulative expected reward and that of the best fixed arm in hindsight over $T$ rounds. In this work, we consider worst-case regret guarantees. For a learning algorithm $A$ and a class of adversaries $\mathfrak{A}$, define $R_{A}(T):=\sup_{\mathrm{Adv}\in\mathfrak{A}} \mathbb{E}_{A,\mathrm{Adv}}\left[\sup_{\pi^*}\sum_{t=1}^T (f_t(\pi^*)-f_t(\pi_t))\right]$. We say a function class $\mathcal{F}$ is learnable against adversaries in $\mathfrak{A}$ if there exists a learning algorithm $A$ such that $R_{A}(T)=o(T)$. Learnability depends on the complexity of the function class. This leads to the central question of this work:

\begin{center}
\emph{Which function class $\mathcal{F}$ is learnable in the adversarial bandit setting?}
\end{center}
Interestingly, a long line of recent literature has investigated learnability in the stochastic bandit setting \citep{amin2011bandits,russo2013eluder,foster2021statistical,foster2023tight,hanneke2023bandit,hanneke2025complete,brukhim2025hardness}. Most recently, \citet{hanneke2025complete} characterize learnability for stochastic bandits using a remarkably simple complexity measure, the \emph{generalized maximin volume}, defined by 
\begin{equation*}
\label{dimension}
    \begin{aligned}
    \gamma_{\mathcal{F},\alpha}=\sup_{p \in \Delta(\Pi)} \inf_{f \in \mathcal{F}} \mathbb{P}_{\pi \sim p}(\sup_{\pi^*} f(\pi^*)-f(\pi)\leq \alpha).
    \end{aligned}
\end{equation*}
By comparison, learnability in the adversarial bandit setting remains far less understood. \citet{foster2022complexity} study adversarial bandits within a broader framework of adversarial decision making. They introduce a convexified variant of the decision-estimation coefficient (DEC) and establish upper and lower bounds in terms of this quantity. However, they explicitly assume a finite arm space $\Pi$, and their upper bound contains an additive $\log|\Pi|$ term, which makes the guarantee vacuous for general infinite arm spaces. A complete characterization of learnability for adversarial bandits is therefore still lacking.

In this work, we completely resolve this question. Our main results are based on a convexified version of the \emph{generalized maximin volume} introduced by \citet{hanneke2025complete}. Specifically, rather than applying this quantity directly to the original function class \(\mathcal F\), we apply it to its convex hull, $\operatorname{co}(\mathcal F)
=
\left\{
\sum_{i=1}^N \lambda_i f_i :
f_i \in \mathcal F,\;
\lambda_i \ge 0,\;
N\geq 1,\;
\sum_{i=1}^N \lambda_i = 1
\right\}$.
We denote the resulting complexity measure by
\(\gamma_{\operatorname{co}(\mathcal F),\alpha}\), and show that it plays a central role in the adversarial bandit learnability problem. Our main result is a complete characterization of learnability in the noisy bandits setting, described as follows:
\begin{theorem}[Characterization of learnability]
    The following statements are equivalent:
    \begin{itemize}
        \item $\gamma_{\operatorname{co}(\mathcal{F}),\alpha}>0\;\forall \alpha \in (0,1)$;
        \item $\mathcal{F}$ is learnable against oblivious adversaries;
        \item $\mathcal{F}$ is learnable against adaptive adversaries.
    \end{itemize}
\end{theorem}

The main strategy for oblivious adversaries is straightforward: we sample arms from a distribution witnessing \(\gamma_{\operatorname{co}(\mathcal F),\alpha}>0\). With sufficiently many samples, we obtain, with high probability, a finite arm set containing a near-optimal arm. We can then run an Exp3-type algorithm to compete with this arm. For adaptive adversaries, however, the adversary can adapt to the sampled arms, so a more sophisticated approach is needed. Our main technical innovation for adaptive adversaries is the notion of a distribution cover, which we introduce below.

\begin{definition}[Distribution cover]
    A set of distributions $\mathcal{I}_{\alpha,\beta}$ is an $(\alpha,\beta)$-distribution cover of the function class $\mathcal{F}$ if for any function $f \in \mathcal{F}$, there exists $p \in \mathcal{I}_{\alpha,\beta}$ such that $\mathbb{P}_{\pi \sim p}(\sup_{\pi^*} f(\pi^*)-f(\pi)\leq \alpha)\geq 1-\beta$. The minimum cardinality of this set $|\mathcal{I}_{\alpha,\beta}|$ is called the distribution covering number.
\end{definition}

The key property of distribution covers is established by the following theorem.

\begin{theorem}
\label{conjecture1}
    For any function class $\mathcal{F}$, $\gamma_{\operatorname{co}(\mathcal{F}),\alpha}>0 \;\forall \alpha \in (0,1)$ implies $\operatorname{co}(\mathcal{F})$ admits a finite $(\alpha,\beta)$-distribution cover  $\forall \alpha,\beta \in (0,1)$. 
\end{theorem}

Given such a finite distribution cover, we can again apply an Exp3-type algorithm, treating each distribution in the cover as a meta-arm.




The remainder of the paper is organized as follows. Section \ref{example} presents several illustrative examples of our characterization. Section \ref{obliviou} characterizes learnability against oblivious adversaries. Section \ref{countable} studies learnability in countable arm spaces. Section \ref{general} extends these results to general, possibly uncountable, arm spaces. Finally, Section \ref{conclusion} concludes with several open questions motivated by our work. Throughout the paper, we assume that all functions under consideration are measurable and take values in $[0,1]$.


\section{Example}
\label{example}
The following examples illustrate our characterization and the role of
convexification.

\begin{example}[Lipschitz class]
Suppose $\Pi$ is a compact subset of $\mathbb{R}^d$ and every function in
$\operatorname{co}(\mathcal F)$ is $L$-Lipschitz for a common constant
$L>0$. For every $\alpha>0$, there exists a finite $\alpha/L$-net $\mathcal{H}_{\alpha}$ such that every arm lies within distance $\alpha/L$ of some arm in $\mathcal{H}_{\alpha}$. Thus, the uniform distribution over $\mathcal{H}_\alpha$ assigns probability at
least $1/|\mathcal{H}_\alpha|$ to the $\alpha$-optimal arms of every such $f$.
 Hence, $\mathcal F$ is learnable against
adaptive adversaries. This includes linear reward classes with uniformly
bounded parameters on compact finite-dimensional arm spaces.
\end{example}

\begin{example}[Singleton]
Suppose $\Pi$ is infinite and
$\mathcal F$ contains every single-spike function
$f_a(\pi)=\mathbf{1}_{\{\pi=a\}}$, $a\in\Pi$.
For any distribution $p$ over $\Pi$, we have
$\inf_{a\in\Pi}p(\{a\})=0$. Since an optimal spike can be placed on an arm with arbitrarily small sampling mass of $p$. Thus $\gamma_{\operatorname{co}(\mathcal{F}),\alpha}=0$ for $\alpha \in (0,1)$, and the class is not learnable.
\end{example}

\begin{example}[Finite-arm bandit]
Suppose $|\Pi|=K<\infty$. For every $f\in \operatorname{co}(\mathcal F)$, the uniform distribution assigns probability at least $\frac{1}{K}$ to an optimal arm, so $\gamma_{\operatorname{co}(\mathcal F),\alpha}\ge \frac1K$ for every $\alpha>0$. Thus, every bounded finite-arm class is learnable, recovering the classical finite-arm case.
\end{example}

\begin{example}[Stochastic versus adversarial separation]
Consider binary function class $\mathcal{F} := \{ f : [0,1] \to \{0,1\} | \int_0^1 f(\pi) d\pi = \frac{1}{2} \}$. Then $\gamma_{\mathcal F,\alpha}=\frac{1}{2}$ for every $\alpha\in(0,1)$. However, for any distribution $p$, one can choose an arm $\pi'$ with $p(\{\pi'\})=0$ and $f_1,f_2\in\mathcal F$ such that $\frac{f_1+f_2}{2}=1$ for $\pi=\pi'$ and $\frac{f_1+f_2}{2}=\frac{1}{2}$ for $\pi\neq \pi'$. Hence $\gamma_{\operatorname{co}(\mathcal F),\alpha}=0$ for $\alpha\in(0,\frac{1}{2})$. This example separates stochastic and adversarial learnability.
\end{example}

\section{Learnability against oblivious adversaries}
\label{obliviou}

An adversary is oblivious if she chooses the entire sequence of functions before the game starts, thus independent of the learner’s actions. First, we introduce our characterization of learnability against oblivious adversaries (Theorem \ref{oblivious_adversary}). 

\begin{theorem}
\label{oblivious_adversary}
  $\mathcal{F}$ is learnable against oblivious adversaries if and only if $\gamma_{\operatorname{co}(\mathcal{F}),\alpha}>0\; \forall \alpha \in (0,1)$.
\end{theorem}


Before moving into the high-level ideas of the proof, we first remind the reader of the typical algorithms for the adversarial bandit problem with finite arms. We use a variant of the Exp3 algorithm that provides high-probability guarantee, Exp3-IX algorithm. The Exp3-IX algorithm \citep{neu2015explore} works by maintaining a weight for each available arm, and updating these weights according to the observed rewards. Since the algorithm only observes the reward of the arm selected at each round, it uses an importance weight estimator to obtain reward estimates for all arms. It then applies a multiplicative weight update rule to adjust the weights of the arms accordingly. (See Appendix \ref{exp3_section} for further details.)

The sufficiency direction is relatively direct given the complexity measure: whenever $\gamma_{\operatorname{co}(\mathcal{F}),\alpha}>0$, one can sample sufficiently many arms so that, with high probability, the resulting finite set contains an arm that is nearly optimal for the average reward function, and then apply the Exp3-IX algorithm on this finite subset of arms, which provides guarantee relative to the best of these fixed arms.


The necessity direction builds on techniques from the analysis of stochastic noisy bandits by reducing to the Bernoulli-noise setting~\citep{hanneke2025complete}. This reduction allows us, at the cost of an exponential loss in the quantitative bound, to construct a distribution witnessing that $\gamma_{\operatorname{co}(\mathcal{F}),\alpha}$ is positive.


\begin{algorithm2e}[htbp]
	\caption{Learning algorithm for the function class $\mathcal{F}$ against oblivious adversaries}
    \label{upper_bound_algo}
	\KwIn{Time horizon $T$, Parameter $\alpha$}
 
    Let $p$ be any distribution satisfying $\frac{\gamma_{\operatorname{co}(\mathcal{F}),\alpha}}{2}\leq\inf_{f \in \operatorname{co}(\mathcal{F})} \mathbb{P}_{\pi \sim p}\left(\sup_{\pi^*}f(\pi^*)-f(\pi)\leq \alpha\right)$.
    
    Sample $m=\frac{2}{\gamma_{\operatorname{co}(\mathcal{F}),\alpha}}\log T$ arms independently from $p: \pi'_1,\pi'_2,...,\pi'_m$.
    
    Run the Exp3-IX algorithm on these $m$ arms for $T$ rounds.
\end{algorithm2e}
\begin{theorem}
\label{upper_bound_learnability}
    $\gamma_{\operatorname{co}(\mathcal{F}),\alpha}>0\; \forall \alpha \in (0,1)$ is a sufficient condition for learnability of $\mathcal{F}$ against oblivious adversaries. 
\end{theorem}
\begin{proof} 
We want to show if $\gamma_{\operatorname{co}(\mathcal{F}),\alpha}>0\; \forall \alpha \in (0,1)$, then the function class $\mathcal{F}$ is learnable against oblivious adversaries. Fix any horizon $T$ and parameter $\alpha>0$. Let $f_1, \dots, f_T \in \mathcal{F}$ be the sequence of reward functions chosen by the oblivious adversary. Since the adversary is oblivious, we can consider this sequence of functions to be fixed in advance. Define the average reward function  $\bar f:= \frac{1}{T} \sum_{t=1}^T f_t$. By convexity, we have $\bar f \in \operatorname{co}(\mathcal{F})$.

Next, recall that distribution $p$ satisfies $\frac{\gamma_{\operatorname{co}(\mathcal{F}),\alpha}}{2}\leq \inf_{f \in \operatorname{co}(\mathcal{F})} \mathbb{P}_{\pi \sim p}\left(\sup_{\pi^*}f(\pi^*)-f(\pi)\leq \alpha\right)$. Apply this to $\bar f$, we have $\mathbb{P}_{\pi \sim p}\left(\sup_{\pi^*}\bar{f}(\pi^*)-\bar{f}(\pi)\leq \alpha\right)\geq \frac{\gamma_{\operatorname{co}(\mathcal{F}),\alpha}}{2}>0$. When we sample $\pi$ from distribution $p$ for $m=\frac{2}{\gamma_{\operatorname{co}(\mathcal{F}),\alpha}}\log T$ times, we have:
\begin{equation*}
    \begin{aligned}
    &\mathbb{P}\left(\exists \pi_i: \sup_{\pi^*}\bar{f}(\pi^*)-\bar{f}(\pi_i)\leq \alpha \right)\\
    = &1-\mathbb{P}\left(\forall \pi_i:  \sup_{\pi^*}\bar{f}(\pi^*)-\bar{f}(\pi_i)> \alpha \right)\\
    \geq &1-\left(1-\frac{\gamma_{\operatorname{co}(\mathcal{F}),\alpha}}{2}\right)^{\frac{2}{\gamma_{\operatorname{co}(\mathcal{F}),\alpha}}\log T}\\
    &\geq  1-e^{-\log T} \\
    &\geq 1-\frac{1}{T}
    \end{aligned}
\end{equation*}

Namely, with probability at least $1-\frac{1}{T}$, there exists an arm $
\tilde\pi$ that is $\alpha$-optimal  such that $\sup_{\pi^*} \bar{f}(\pi^*)-\bar{f}(\tilde\pi) \leq \alpha$.  We use event $\mathcal{B}$ to denote that such a $\tilde{\pi}$ exists. Next, we run the Exp3-IX algorithm on this finite set of arms. The regret is bounded by  Lemma \ref{exp3_guarantee} in the Appendix:
    \begin{align}
        R_{A}(T)&=\mathbb{E}\left[\sup_{\pi^*}\sum_{t=1}^T (f_t(\pi^*)-f_t(\pi_t))\middle|\mathcal{B}\right]P(\mathcal{B})+\mathbb{E}\left[\sup_{\pi^*}\sum_{t=1}^T (f_t(\pi^*)-f_t(\pi_t))\middle|\bar{\mathcal{B}}\right]P(\bar{\mathcal{B}})\notag\\
        &\leq \mathbb{E}\left[\sup_{\pi^*}\sum_{t=1}^T (f_t(\pi^*)-f_t(\pi_t))\middle|\mathcal{B}\right]+\frac{1}{T}\cdot T\notag\\
        &=\mathbb{E}\left[\sup_{\pi^*}\sum_{t=1}^T (f_t(\pi^*)-f_t(\tilde\pi)+f_t(\tilde\pi)-f_t(\pi_t))\middle|\mathcal{B}\right]+1\notag\\
        &\leq \mathbb{E}\left[\sup_{\pi^*}\sum_{t=1}^T (f_t(\pi^*)-f_t(\tilde\pi))+\sum_{t=1}^T (f_t(\tilde\pi)-f_t(\pi_t))\middle|\mathcal{B}\right]+1\notag\\
        &\leq \alpha T+2\sqrt{2mT\log m}+1\notag\\
        &\leq \alpha T+4\sqrt{\frac{1}{\gamma_{\operatorname{co}(\mathcal{F}),\alpha}}T\log (T)\log \left(\frac{2}{\gamma_{\operatorname{co}(\mathcal{F}),\alpha}}\log (T)\right)}+1. \label{eqn:last-expression}
    \end{align}




    For simplicity of presentation, we define $g(\alpha,T):=4\sqrt{\frac{1}{\gamma_{\operatorname{co}(\mathcal{F}),\alpha}T}\log (T)\log (\frac{2}{\gamma_{\operatorname{co}(\mathcal{F}),\alpha}}\log (T))}$. We want to show by choosing appropriate $\alpha$ depending on $T$, the expression \eqref{eqn:last-expression} above is sublinear in $T$, which is equivalent to showing:
    $$\inf_{\alpha}\alpha +g(\alpha,T)+\frac{1}{T} \rightarrow 0$$


Next, we have $g(\alpha,T)\to0$ for every fixed $\alpha$ when $T \rightarrow \infty$. Now fix $\varepsilon>0$. Choose $\alpha<\frac{\varepsilon}{2}$. Since $\gamma_{\operatorname{co}(\mathcal{F}),\alpha}$ is fixed for a given $\alpha$, there exists $T_0$ such that for all
$T\ge T_0$, $g(\alpha,T)+\frac{1}{T}\leq \frac{\varepsilon}{2}$. Therefore, for all $T\ge T_0$, $\frac{R_A(T)}{T}
\leq \varepsilon$.
Since $\varepsilon>0$ was arbitrary, we conclude that
$\frac{R_A(T)}{T}\to 0$ as $T \to \infty$.
Therefore, Algorithm \ref{upper_bound_algo} has sublinear regret in $T$.
\end{proof}

\begin{theorem}
\label{lower_bound_learnability}
    $\gamma_{\operatorname{co}(\mathcal{F}),\alpha}>0\; \forall \alpha \in (0,1)$ is a necessary condition for learnability of $\mathcal{F}$ against oblivious adversaries.
\end{theorem}
\begin{proof}
We want to show if a function class $\mathcal{F}$ is learnable against oblivious adversaries, it must be $\gamma_{\operatorname{co}(\mathcal{F}),\alpha}>0\; \forall \alpha \in (0,1)$. Supposing $\mathcal{F}$ is learnable, let $A$ be any no-regret learning algorithm for $\mathcal{F}$. 
Define a random variable $\hat{\pi}$ as follows.
We execute the algorithm $A$, but whenever it pulls an arm, we respond with an independent $\mathrm{Bernoulli}(\frac{1}{2})$ reward. 
Let $\pi_1,\ldots,\pi_T$ be the (random) arms pulled by this execution of the algorithm.  
Conditioned on $\pi_1,\ldots,\pi_T$, sample $\hat{\pi}$ by $\text{Uniform}(\pi_1,...,\pi_T)$.
Denote by $p$ the induced marginal distribution of this random variable $\hat{\pi}$.
We will argue this distribution $p$ witnesses $\gamma_{\operatorname{co}(\mathcal{F}),\alpha}>0\;\forall \alpha \in (0,1)$. 

Fix $\alpha \in (0,1)$.
To show $\gamma_{\operatorname{co}(\mathcal{F}),\alpha} > 0$, 
it suffices to show that for any $\bar{f} \in \operatorname{co}(\mathcal{F})$, 
$\mathbb{P}_{\hat{\pi}}(\sup_{\pi^*} \bar{f}(\pi^*) - \bar{f}(\hat{\pi}) \leq \alpha) > c$ for some $\bar{f}$-independent value $c > 0$.
Toward this end, fix any $\bar{f} = \sum_{i=1}^N \lambda_i f'_i$, for some $N\geq 1$, $\lambda_i\in [0,1]$ subject to $\sum_{i=1}^N \lambda_i=1$ and $f'_i\in \mathcal{F}$.
We imagine running $A$ under an oblivious adversary, which, at each round $t$, independently samples a function $f_t \in\{f'_1,\dots,f'_N\}$ according to the distribution $\{\lambda_i\}_{i \leq N}$, that is, $\mathbb P(f_t=f'_i)=\lambda_i$.

Whenever the algorithm pulls arm $\pi'_t$, give reward $r_t \sim \mathrm{Bernoulli}(f_t(\pi'_t))$. We have $\mathbb{E}[r_t|f_t,\pi'_t]=f_t(\pi'_t)$ and $\mathbb{E}[r_t|\pi'_t]=\mathbb{E}[f_t(\pi'_t)|\pi'_t]=\bar{f}(\pi'_t)\;\forall t \in [T]$.  Since these choices are made a priori, this indeed represents an oblivious adversary.  
Define another random variable $\hat{\pi}'$ as follows.
Denote by $\pi'_1,\ldots,\pi'_T$ the (random) arms pulled by the algorithm $A$ during this execution.  Conditioned on $\pi'_1,\ldots,\pi'_T$, let $\hat{\pi}'$ be sampled $\text{Uniform}(\pi'_1,...,\pi'_T)$.
Fix any arm $\pi^*$ with $\bar{f}(\pi^*) \geq \sup_{\pi} \bar{f}(\pi) - \frac{\alpha}{4}$.
Note that the regret
\begin{align*}
R_{A}(T) & \geq \mathbb{E}\!\left[ \sum_{t=1}^{T} (f_t(\pi^*) - f_t(\pi'_t)) \right] 
= \sum_{t=1}^{T} \mathbb{E}\!\left[ (f_t(\pi^*) - f_t(\pi'_t)) \right]
= \sum_{t=1}^{T} \mathbb{E}\!\left[ \mathbb{E}\!\left[ f_t(\pi^*) - f_t(\pi'_t) \middle| \pi'_t \right] \right]
\\ & = \sum_{t=1}^{T} \mathbb{E}\!\left[ \bar{f}(\pi^*) - \bar{f}(\pi'_t) \right]
= \mathbb{E}\!\left[ \sum_{t=1}^{T} (\bar{f}(\pi^*) - \bar{f}(\pi'_t)) \right]
= T \mathbb{E}\!\left[ \bar{f}(\pi^*) - \bar{f}(\hat{\pi}') \right].
\end{align*}
Since the regret $R_A(T)=o(T)$ against any adversary for no-regret learning algorithm $A$, there exists $T_{\alpha} < \infty$ such that, for every $T \geq T_{\alpha}$, $R_A(T) \leq \frac{\alpha}{4}T$, which (by the above) implies 
$\mathbb{E}_{\hat{\pi}'}[\bar{f}(\pi^*)-\bar{f}(\hat{\pi}')] \leq \frac{\alpha}{4}$.
By definition of $\pi^*$, this further implies 
$\mathbb{E}_{\hat{\pi}'}[\sup_{\pi} \bar{f}(\pi)-\bar{f}(\hat{\pi}')] \leq \frac{\alpha}{2}$. By Markov's Inequality (Lemma \ref{markov_inequality}), for any such $T$,  $\mathbb{P}_{\hat{\pi}'}(\sup_{\pi}\bar{f}(\pi)-\bar{f}(\hat{\pi}')\leq \alpha)\geq \frac{1}{2}$. 

Now we couple the two executions in $T_{\alpha}$ rounds using the same internal randomness of $A$ and a common random index $I \sim \mathrm{Uniform}(\{1,\dots,T_{\alpha}\})$. The independent $\mathrm{Bernoulli}(\frac{1}{2})$ reward sequence matches the reward sequence generated by the above oblivious adversary for the first $T_\alpha$ rounds with probability exactly $2^{-T_\alpha}$, independent of the adversary’s chosen function. On this event, the two executions have identical histories, hence $\pi_t=\pi'_t$ for all $t \le T_\alpha$, and therefore $\hat{\pi}=\hat{\pi}'$. Consequently, 
\[
\mathbb P_{\pi \sim p}\!\left(
\sup_{\pi^*}\bar f(\pi^*) - \bar f(\pi) \le \alpha
\right)
\ge
2^{-T_\alpha}\,
\mathbb P\!\left(
\sup_{\pi^*}\bar f(\pi^*) - \bar f(\hat{\pi}') \le \alpha
\right)
\ge 2^{-T_\alpha-1} > 0.
\]
Since this $2^{-T_\alpha-1}$ lower bound holds $\forall \bar f \in \operatorname{co}(\mathcal F)$,
this establishes that $p$ witnesses $\gamma_{\operatorname{co}(\mathcal F),\alpha} > 0$. 
\end{proof}

\section{Learnability in countable arm spaces}
\label{countable}

An adversary is adaptive if, at each round, she may choose the reward function based on the interaction history so far, including the learner’s past actions and observed rewards. Consequently, an arm that was favorable in earlier rounds may become less favorable later. Algorithm \ref{upper_bound_algo} is no longer guaranteed to work in this setting. We first study this setting when the arm space $\Pi$ is countable. In countable arm spaces, we uncover a simple but illuminating connection between $\gamma_{\operatorname{co}(\mathcal{F}),\alpha}$ and the existence of a simple combinatorial complexity measure, namely the $\alpha$-hitting set. This connection yields a clean characterization of learnability. We first introduce a notion of $\alpha$-hitting set.

\begin{definition}[Hitting Set]
    A set of arms $\mathcal{H}_{\alpha}\subseteq \Pi$ is an $\alpha$-hitting set for the function class $\mathcal{F}$ if for any function $f \in \mathcal{F}$, there exists an arm $\pi \in \mathcal{H}_{\alpha}$ such that $\sup_{\pi^*} f(\pi^*)-f(\pi)\leq \alpha$. The minimum cardinality of this set $|\mathcal{H}_{\alpha}|$ is the size of the hitting set.
\end{definition}

Our key observation is positive $\gamma_{\operatorname{co}(\mathcal{F}),\alpha}$ implies a finite $\alpha$-hitting set $\forall \alpha$, stated as Lemma~\ref{hitting_set_implication}.   
\begin{lemma}
\label{hitting_set_implication}
    For countable arm spaces, if $\gamma_{\mathcal{F},\alpha}>0\;\forall \alpha \in (0,1)$, then $\mathcal{F}$ admits a finite $\alpha$-hitting set  $\forall \alpha \in (0,1)$.
\end{lemma}

\begin{proof}
Fix any $\alpha \in (0,1)$. Let $p$ be a probability distribution over arm space $\Pi$ that satisfies $\frac{\gamma_{\mathcal{F},\alpha}}{2}\leq\inf_{f \in \mathcal{F}} \mathbb{P}_{\pi \sim p}\left(\sup_{\pi^*}f(\pi^*)-f(\pi)\leq \alpha\right)$, and let $p_i := p(\pi_i)$ denote the probability mass assigned to arm $\pi_i$. Since arm space $\Pi$ is countable, we index every arm in $\Pi$ as $\pi_1,\pi_2,...,$ based on probability mass in a non-increasing order. By countability, there exists a finite number $t$, such that $\sum_{i>t}p_i \leq \frac{\gamma_{\mathcal{F},\alpha}}{4}$. It follows that, $\forall f \in \mathcal{F}$, $\sum_{i>t}p_i \mathbb{I}[f(\pi_i)\geq \sup_{\pi} f(\pi)-\alpha] \leq \frac{\gamma_{\mathcal{F},\alpha}}{4}$. On the other hand, by the choice of distribution $p$, $\forall f \in \mathcal{F}$, $\frac{\gamma_{\mathcal{F},\alpha}}{2}\leq\sum_i p_i \mathbb{I}[f(\pi_i)\geq \sup_{\pi} f(\pi)-\alpha]$. This gives us: $\forall f \in \mathcal{F}, \sum_{i \leq t} p_i \mathbb{I}[f(\pi_i)\geq \sup_{\pi} f(\pi)-\alpha] \geq \frac{\gamma_{\mathcal{F},\alpha}}{4}>0$.
Hence,  $\forall f \in \mathcal{F}$,  there exists $\pi_i$ satisfying $i\leq t$ such that $f(\pi_i)\geq \sup_{\pi^*} f(\pi^*)-\alpha$. This shows that the finite set $\{\pi_1,\dots,\pi_t\}$ forms an $\alpha$-hitting set.
\end{proof}

Since Lemma~\ref{hitting_set_implication} holds for any function class $\mathcal{F}$, we may apply it with \(\mathcal{F}\) replaced by \(\operatorname{co}(\mathcal{F})\). Therefore, the lemma above shows that positivity of $\gamma_{\operatorname{co}(\mathcal{F}),\alpha}$ guarantees the existence of a finite $\alpha$-hitting set for $\operatorname{co}(\mathcal{F})$. This immediately yields a learning algorithm obtained by running the Exp3-IX algorithm on that finite set of arms, as stated in Theorem~\ref{countable_sufficient}.


\begin{algorithm2e}
	\caption{Learning algorithm for function class $\mathcal{F}$ against adaptive adversary via hitting set}
    \label{upper_bound_algo2}
	\KwIn{Parameter $\alpha$, Time horizon $T$}

    Let $\mathcal H_{\alpha}
= \{\pi_1,\ldots,\pi_m\}$ be a finite $\alpha$-hitting set for $\operatorname{co}(\mathcal F)$;\\
Run the Exp3-IX algorithm on the finite arm set $\mathcal H_{\alpha}$ for $T$ rounds.
\end{algorithm2e}
\begin{theorem}
\label{countable_sufficient}
    If $\operatorname{co}(\mathcal{F})$ admits a finite $\alpha$-hitting set $\forall \alpha \in (0,1)$, then the function class $\mathcal{F}$ is learnable against adaptive adversaries.
\end{theorem}

\begin{proof}
Define $\tilde \pi:=\arg\max_{\pi \in \mathcal{H}_{\alpha}} \frac{1}{T} \sum_{t=1}^T f_t(\pi)$. We decompose the regret of Algorithm \ref{upper_bound_algo2} as follows:
\begin{equation*}
    \begin{aligned}
        R_A(T)&=\mathbb{E}\left[\sup_{\pi^*}\sum_{t=1}^T (f_t(\pi^*)-f_t(\pi_t))\right]\\
        &=\mathbb{E}\left[\sup_{\pi^*}\sum_{t=1}^T (f_t(\pi^*)-f_t(\tilde\pi)+f_t(\tilde\pi)-f_t(\pi_t))\right]\\       
        &\leq \mathbb{E}\left[\sup_{\pi^*}\sum_{t=1}^T (f_t(\pi^*)-f_t(\tilde\pi))\right]+\mathbb{E}\left[\sum_{t=1}^T (f_t(\tilde\pi)-f_t(\pi_t))\right]\\
    \end{aligned}
\end{equation*}
Define $\bar f=\frac{1}{T}\sum_{t=1}^T f_t \in \operatorname{co}(\mathcal{F})$, by the hitting set property, there exists $\pi' \in\mathcal H_{\alpha}$ such that $\sup_{\pi^*} \bar f(\pi^*)- \bar f(\pi') \leq \alpha$. Since $\tilde \pi$ maximizes $\bar f(\pi)$ over $ \pi \in\mathcal H_{\alpha}$, it follows that $\bar f(\tilde \pi) \ge \bar f(\pi')\geq\sup_{\pi^*}\bar f(\pi^*)-\alpha$. Thus, the first term is bounded by $\alpha T$. The second term is bounded by Exp3-IX algorithm guarantee. Therefore,
$$R_A(T) \leq \alpha T+2\sqrt{2|\mathcal{H}_{\alpha}|T\log |\mathcal{H}_{\alpha}|}$$

    For simplicity of presentation, we define $g(\alpha,T):=2\sqrt{\frac{2|\mathcal{H}_{\alpha}|\log |\mathcal{H}_{\alpha}|}{T}}$. We want to show by choosing appropriate $\alpha$ depending on $T$, the expression above is sublinear in $T$, which is equivalent to showing:
    $$\inf_{\alpha}\alpha +g(\alpha,T) \rightarrow 0$$
Next, we have $g(\alpha,T)\to0$ for every fixed $\alpha$ when $T \rightarrow \infty$. Now fix $\varepsilon>0$. Choose $\alpha<\frac{\varepsilon}{2}$. Since $|\mathcal{H}_{\alpha}|$ is fixed for a given $\alpha$, there exists $T_0$ such that for all
$T\ge T_0$, $g(\alpha,T)\leq \frac{\varepsilon}{2}$. Therefore, for all $T\ge T_0$, $\frac{R_A(T)}{T}
\leq \varepsilon$.
Since $\varepsilon>0$ was arbitrary, we conclude that
$\frac{R_A(T)}{T}\to 0$ as $T \to \infty$.
Therefore, Algorithm \ref{upper_bound_algo2} has sublinear regret in $T$.

\end{proof}


Together with Theorem \ref{oblivious_adversary}, we obtain the following characterization of learnability in countable arm spaces (Theorem \ref{countable_learnability}).

\begin{theorem}
\label{countable_learnability}
Suppose the arm space \(\Pi\) is countable. Then the following statements are equivalent:
    \begin{itemize}
        \item  $\mathcal{F}$ is learnable against oblivious adversaries;
        \item  $\mathcal{F}$ is learnable against adaptive adversaries;
        \item $\gamma_{\operatorname{co}(\mathcal{F}),\alpha}>0\; \forall \alpha \in (0,1)$;
        \item $\operatorname{co}(\mathcal{F})$ admits a finite \(\alpha\)-hitting set $\forall \alpha \in (0,1)$.
    \end{itemize}
\end{theorem}

\section{Learnability in General Arm Spaces}
\label{general}

Next, we present our results for the general arm space setting. If the arm space $\Pi$ is uncountable, the notion of hitting set is no longer appropriate for characterizing learnability in adversarial bandits. In particular, Example~\ref{example1} below exhibits a learnable function class with distribution covering number 1, while every hitting set has infinite size.

\begin{example}
\label{example1}
Consider the convex function class $\mathcal{F}_1 := \left\{ f : [0,1] \to [0,1] \;\middle|\; \int_0^1 f(\pi)\, d\pi = 1 \right\}$. That is, $\mathcal{F}_1$ consists of all functions that equal $1$ almost everywhere on $[0,1]$. For this function class, the size of any hitting set is infinite. Since for any finite set $\mathcal{H} \subset [0,1]$, there exists a function $f \in \mathcal{F}_1$ satisfying $f(\pi)=0$ for any $\pi \in \mathcal{H}$ but have value $1$ elsewhere.

In contrast, the size of  $(\alpha,\beta)$-distribution cover is $1$. Let $p$ be the uniform distribution over $[0,1]$. For every $f \in \mathcal{F}_1$, we have $f(\pi)=1$ almost every $\pi$. Thus, $\mathbb{P}_{\pi \sim p}\bigl(\sup_{\pi^*} f(\pi^*) - f(\pi) \le \alpha \bigr)=1$
for every $\alpha > 0$. Hence the single distribution $\{p\}$ already forms an $(\alpha,0)$-distribution cover. This example shows that allowing distributions, rather than only deterministic arms, is essential in the definition of the complexity measure.
\end{example}

To address the issue illustrated by the example above, we use a generalization of the hitting set, called a \emph{distribution cover}, which was introduced in Section \ref{intro}. The hitting set can be viewed as a special version of the distribution cover, which has $\beta=0$. Consequently, any function class $\mathcal{F}$ that admits a finite hitting set also admits a finite distribution cover. On the other hand, finite distribution covers are still sufficiently structured to support learning algorithms based on the Exp3-IX algorithm, leading to the following sufficient condition.

\begin{algorithm2e}[htbp]
	\caption{Learning algorithm for the function class $\mathcal{F}$ against adaptive adversaries via distribution cover}
    \label{upper_bound_algo3}
	\KwIn{Parameter $\alpha,\beta$, Time horizon $T$}
    Let $\mathcal{I}_{\alpha,\beta}
= \{p_1,\ldots,p_m\}$ be a finite $(\alpha,\beta)$-distribution cover for $\operatorname{co}(\mathcal F)$;\\
    Run the Exp3-IX algorithm on $\mathcal{I}_{\alpha,\beta}$ for $T$ rounds: at each round $t$, if Exp3-IX selects meta-arm $p_i$, sample $\pi_t \sim p_i$ and play arm $\pi_t$. Observe the reward $r_t$ and use it as the feedback for meta-arm $p_i$ in the Exp3-IX.
\end{algorithm2e}

\begin{theorem}
\label{distribution_cover}
   If $\operatorname{co}(\mathcal{F})$ admits a finite $(\alpha,\beta)$-distribution cover $\forall \alpha,\beta \in (0,1)$, then the function class $\mathcal{F}$ is learnable against adaptive adversaries.
\end{theorem}

\begin{proof}
First, define $\tilde{p}:=\arg\max_{p \in \mathcal{I}_{\alpha,\beta}}\frac{1}{T}\sum_{t=1}^T \mathbb{E}_{\pi \sim p}[f_t(\pi)]$. We decompose the regret of Algorithm \ref{upper_bound_algo3} as follows:
\begin{equation*}
    \begin{aligned}
        R_A(T)&=\mathbb{E}\left[\sup_{\pi^*}\sum_{t=1}^T (f_t(\pi^*)-f_t(\pi_t))\right]\\
        &=\mathbb{E}\left[\sup_{\pi^*}\sum_{t=1}^T (f_t(\pi^*)-\mathbb{E}_{\pi\sim \tilde{p}}[f_t(\pi)]+\mathbb{E}_{\pi\sim \tilde{p}}[f_t(\pi)]-f_t(\pi_t))\right]\\
        &\leq \mathbb{E}\left[\sup_{\pi^*}\sum_{t=1}^T (f_t(\pi^*)-\mathbb{E}_{\pi \sim \tilde{p}}[f_t(\pi)])\right]+\mathbb{E}\left[\sum_{t=1}^T (\mathbb{E}_{\pi \sim \tilde{p}}[f_t(\pi)]-f_t(\pi_t))\right]
    \end{aligned}
\end{equation*}

Define $\bar f=\frac{1}{T}\sum_{t=1}^T f_t \in \operatorname{co}(\mathcal{F})$, by the $(\alpha,\beta)$-distribution cover property, there exists $p^\star\in\mathcal I_{\alpha,\beta}$ such that $\mathbb{P}_{\pi \sim p^*}(\sup_{\pi^*}\bar f(\pi^*)-\bar f(\pi) \leq \alpha) \geq 1-\beta$. It follows that
\begin{equation*}
    \begin{aligned}
        \mathbb{E}_{\pi \sim  p^*}\left[\sup_{\pi^*}\bar  f(\pi^*)- \bar f(\pi) \right] \leq \mathbb{P}_{\pi \sim  p^*}&\left(\sup_{\pi^*}\bar f(\pi^*)-\bar f(\pi) \leq \alpha \right)\cdot\alpha+\\
        &\mathbb{P}_{\pi \sim  p^*}\left(\sup_{\pi^*}\bar f(\pi^*)-\bar f(\pi) > \alpha \right)\cdot 1 \leq \alpha+\beta.
    \end{aligned}
\end{equation*}
Since $\tilde p$ maximizes $\mathbb E_{\pi\sim p}[\bar f(\pi)]$ over $p\in\mathcal I_{\alpha,\beta}$, $\mathbb E_{\pi\sim \tilde p}[\bar f(\pi)]
\ge
\mathbb E_{\pi\sim p^\star}[\bar f(\pi)]
\ge
\sup_{\pi^*}\bar f(\pi^*)-(\alpha+\beta)$. Thus, the first term is bounded by $(\alpha+\beta)T$. The second term is the regret of Exp3-IX algorithm over the finite set $\mathcal{I}_{\alpha,\beta}$, which yields
        $$R_A(T)\leq(\alpha +\beta)T+2\sqrt{2|\mathcal{I}_{\alpha,\beta}|T\log |\mathcal{I}_{\alpha,\beta}|}$$
Define $g(\alpha,\beta,T):=2\sqrt{\frac{2|\mathcal{I}_{\alpha,\beta}|\log |\mathcal{I}_{\alpha,\beta}|}{T}}$, it remains to choose $\alpha,\beta>0$ as a function of $T$ so that the bound is sublinear in $T$, which is equivalent to show:
    $$\inf_{\alpha,\beta}\alpha+\beta +g(\alpha,\beta,T) \rightarrow 0$$

Since for any  fixed $\alpha,\beta>0$, the quantity $|\mathcal{I}_{\alpha,\beta}|$ is finite. Hence, $g(\alpha,\beta,T) \to 0$ as $T \to \infty$. Now fix any $\varepsilon>0$. Choose $\alpha,\beta>0$ such that
$\alpha+\beta \le \frac{\varepsilon}{2}$.
Since $|\mathcal I_{\alpha,\beta}|<\infty$, there exists $T_0$ such that for all
$T\ge T_0$, $g(\alpha,\beta,T)\leq \frac{\varepsilon}{2}$. Therefore, for all $T\ge T_0$, $\frac{R_A(T)}{T}
\leq \varepsilon$.
Since $\varepsilon>0$ was arbitrary, we conclude that
$\frac{R_A(T)}{T}\to 0$ as $T \to \infty$.
Therefore, Algorithm \ref{upper_bound_algo3} has sublinear regret in $T$.
\end{proof}

It is natural to ask whether an analogue of Lemma~\ref{hitting_set_implication} still holds for uncountable arm spaces when hitting sets are replaced by distribution covers. We show that this is indeed the case.



\begin{theorem}[Restatement of Theorem \ref{conjecture1}]
    For any function class $\mathcal{F}$, $\gamma_{\operatorname{co}(\mathcal{F}),\alpha}>0 \;\forall \alpha \in (0,1)$ implies $\operatorname{co}(\mathcal{F})$ admits a finite $(\alpha,\beta)$-distribution cover  $\forall \alpha,\beta \in (0,1)$. 
\end{theorem}
\begin{proof}
Let $\mathcal{G}=\operatorname{co}(\mathcal{F})$ and assume all the functions in $\mathcal{G}$ are measurable. Let $A_g^r= \{\pi \in \Pi : \sup_{\pi^*}g(\pi^*)-g(\pi) \le r\}$. Fix \(\alpha,\beta \in (0,1)\) and set $\rho=\frac{\alpha\beta}{8}$ and $\delta=\frac{\gamma_{\mathcal{G},\rho}}{2}$. Since $\gamma_{\mathcal{G},\rho} > 0$,
choose a distribution \(p \in \Delta(\Pi)\) such that $p(A_g^\rho) \ge \delta\; \forall g \in \mathcal{G}$.

Suppose, toward contradiction, that \(\mathcal{G}\) has no finite \((\alpha,\beta)\)-distribution cover. Then for every finite family of distributions \(q_1,\dots,q_N\), there exists \(g \in \mathcal{G}\) such that $q_i(A_g^\alpha) < 1-\beta,\;\forall i=1,\dots,N$. Now we fix $m$ to be chosen large enough later. We construct functions \(g_1,\dots,g_m \in \mathcal{G}\) inductively: Suppose \(g_1,\dots,g_{t-1}\) have already been chosen. Define indicator function $Y_s(\pi) := \mathbb{I}\{\pi \in A_{g_s}^\alpha\}$.
The variables \(Y_1,\dots,Y_{t-1}\) generate a finite partition of \(\Pi\) into atoms \(C\). For every atom \(C\) with \(p(C)>0\), define the conditional distribution $q_C(B) = \frac{p(B \cap C)}{p(C)}$.
There are only finitely many such distributions \(q_C\). By assumption, choose \(g_t \in \mathcal{G}\) such that $q_C(A_{g_t}^\alpha) < 1-\beta$,
for every positive-\(p\)-mass atom \(C\). Equivalently, $\mathbb E_p[Y_t \mid Y_1,\dots,Y_{t-1}] \le 1-\beta \; a.s.$

Now define the smaller-scale indicators $Z_t(\pi) := \mathbf 1\{\pi \in A_{g_t}^\rho\}$.
By assumption, $\mathbb E_p Z_t=p(A_{g_t}^\rho) \ge \delta\;\forall t$. We then have $\mathbb E_p \sum_{t=1}^m Z_t
\ge \delta m$.
Hence there exists some \(\pi_0 \in \Pi\) such that $\sum_{t=1}^m Z_t(\pi_0) \ge \delta m$.
Let $S := \{t \le m : \pi_0 \in A_{g_t}^\rho\}$.
Then $|S| \ge \delta m$. Namely, $\pi_0$ is $\rho$-optimal for all functions in $S$. Now average the functions indexed by \(S\), $h := \frac{1}{|S|}\sum_{t \in S} g_t$. Since \(\mathcal{G}\) is convex, we have \(h \in \mathcal{G}\).

Next, we will show within set $S$, if an arm $\pi$ is $\rho$-optimal for the averaged function $h$, then $\pi$ must be $\alpha$-optimal for almost all of the individual functions $g_t$ in $S$. Since both $\pi$ and $\pi_0$ are $\rho$-optimal for average function $h$, we have $|h(\pi)-h(\pi_0)|\leq \rho$. Equivalently, $|\sum_t g_t(\pi)-g_t(\pi_0)|\leq \rho|S|$. In addition, $\pi_0$ is $\rho$-optimal for any function $g_t$ in $S$, namely, $|\sup_{\pi^*}g_t(\pi^*)-g_t(\pi_0)|\leq \rho\;\forall t$. By triangle inequality, we have 
$$\frac{1}{|S|}\sum_{t \in S}(\sup_{\pi^*}g_t(\pi^*)-g_t(\pi)) \leq \frac{1}{|S|}\sum_{t\in S} (\sup_{\pi^*}g_t(\pi^*)-g_t(\pi_0)+g_t(\pi_0)-g_t(\pi))\leq 2\rho=\frac{\alpha \beta}{4}$$

By average argument, we then have $\frac{1}{|S|}
\left|
\{t \in S : \sup_{\pi^*}g_t(\pi^*)-g_t(\pi) > \alpha\}
\right|
\le
\frac{\alpha\beta/4}{\alpha}
=
\frac{\beta}{4}$. Equivalently, $\frac{1}{|S|}\sum_{t \in S} Y_t(\pi)
\ge
1-\frac{\beta}{4}.$ Therefore, we have $A_h^\rho
\subseteq
\left\{
\pi :
\sum_{t \in S} Y_t(\pi)
\ge
\left(1-\frac{\beta}{4}\right)|S|
\right\}$.

It remains to bound the probability of the event on the right. Let $\mu= 1-\beta$, $\theta= 1-\frac{\beta}{4}$.
Then \(\theta > \mu\). Recall $\mathbb E_p[Y_t \mid Y_1,\dots,Y_{t-1}]
\le \mu\;\forall t$. For any deterministic subset \(S \subseteq \{1,\dots,m\}\), the same conditional expectation domination holds along the subsequence indexed by \(S\). Therefore, by martingale chernoff bound (Lemma~\ref{martingale_chernoff_bound}),
\[
p\left(
\sum_{t \in S}Y_t \ge \theta |S|
\right)
\le
\exp\left(-D(\theta\|\mu)|S|\right),
\]
\looseness=1 where $D(\theta\|\mu)=\theta \log\frac{\theta}{\mu}
+
(1-\theta)\log\frac{1-\theta}{1-\mu}$ is the binary relative entropy.
Now set $c_\beta={}$\linebreak
$D(1-\beta/4\,\|\,1-\beta)>0$. We obtain $p(A_h^\rho)
\le
\exp(-c_\beta |S|)
\le
\exp(-c_\beta \delta m)$.
Choose \(m=
        \left\lfloor
        \frac{\log(1/\delta)}{c_\beta\delta}
    \right\rfloor+1\) such that $\exp(-c_\beta \delta m) < \delta$.
Then $p(A_h^\rho) < \delta$. This is a contradiction. Therefore, \(\mathcal{G}\) must admit a finite \((\alpha,\beta)\)-distribution cover. Since this analysis works for any $\alpha,\beta \in (0,1)$, this completes the proof.
\end{proof}

\paragraph{Quantitative bound}
The proof of Theorem~\ref{conjecture1} also yields an explicit bound on the
distribution covering number. Fix $\alpha,\beta\in(0,1)$ and write
$\Gamma:=\gamma_{\operatorname{co}(\mathcal F),\alpha\beta/8}>0$.
Then we have $|\mathcal I_{\alpha,\beta}|
    \le
    \exp\!\left(
        \frac{8\log(2/\Gamma)}{\beta\Gamma}
    \right)$. Since the first $t-1$ binary indicators in the proof generate at most
$2^{t-1}$ partition atoms, the resulting cover has at most
$2^{m-1}$ elements.
Finally, by definition of binary relative entropy, $c_\beta
    \ge \frac{\beta}{4}$. This
implies $\log |\mathcal I_{\alpha,\beta}|
    \le
    \frac{2\log 2}{\Gamma c_\beta}\log\frac{2}{\Gamma}
    \le
    \frac{8\log(2/\Gamma)}{\beta\Gamma}$, which proves $|\mathcal I_{\alpha,\beta}|
    \le
    \exp\!\left(
        \frac{8\log(2/\Gamma)}{\beta\Gamma}
    \right)$.

We remark that the opposite direction follows immediately from our proven results: that is, since finite distribution covers suffice for learnability (Theorem~\ref{distribution_cover}) and $\gamma_{\operatorname{co}(\mathcal{F}),\alpha} > 0$ is necessary for learnability (Theorem~\ref{oblivious_adversary}), we have that if $\operatorname{co}(\mathcal{F})$ admits a finite $(\alpha,\beta)$-distribution cover $\forall \alpha,\beta \in (0,1)$ then $\gamma_{\operatorname{co}(\mathcal{F}),\alpha} > 0$ $\forall \alpha \in (0,1)$. Theorem~\ref{conjecture1} has the following immediate corollary.

\begin{corollary}[Characterization of learnability]
\label{corollary1}
    The following statements are equivalent:
    \begin{itemize}
        \item  $\mathcal{F}$ is learnable against oblivious adversaries;
        \item  $\mathcal{F}$ is learnable against adaptive adversaries;
        \item $\gamma_{\operatorname{co}(\mathcal{F}),\alpha}>0\; \forall \alpha \in (0,1)$;
        \item $\operatorname{co}(\mathcal{F})$ admits a finite $(\alpha,\beta)$-distribution cover  $\forall \alpha,\beta \in (0,1)$.
    \end{itemize}
\end{corollary}

\paragraph{Remarks on the role of convexity in Theorem~\ref{conjecture1}}
Our Theorem~\ref{conjecture1} states that $\gamma_{\operatorname{co}(\mathcal F),\alpha}>0$ implies there exist finite distribution covers.  We remark that a proof of this theorem must necessarily rely on the convexity of $\operatorname{co}(\mathcal F)$, as we can show that it fails to hold for general non-convex function classes $\mathcal F$.
Specifically, Example~\ref{example2} below shows that there exists a function class $\mathcal{F}$ with $\gamma_{\mathcal F,\alpha}>0$,  $\gamma_{\operatorname{co}(\mathcal F),\alpha}=0$ and with infinite distribution covering number. This shows that convexity must play an essential role in the proof. In other words, Theorem~\ref{conjecture1} can hold, at best, only for convex function classes.

\begin{example}
\label{example2}
Let the arm space be $\Pi=[0,1]$ equipped with the Lebesgue measure $\lambda$. Consider the ternary function class
$\mathcal F_2 := \left\{ f:\Pi\to\{0,\frac{1}{2},1\}  \middle|\lambda\bigl(\{f(\pi)=1\}\bigr)
=\lambda\bigl(\{f(\pi)=0\}\bigr)
=\frac13
\right\}$.

We first show $\gamma_{\mathcal F_2,\alpha} > 0$. For any $f \in \mathcal F_2$, exactly $\frac{1}{3}$ of the domain satisfies $f(\pi)=1$. Thus, under the uniform distribution $u$ over $\Pi$, $\mathbb P_{\pi \sim u}\bigl(f(\pi)=1\bigr)=\tfrac13$. Hence we have $\gamma_{\mathcal F_2,\alpha} \ge \tfrac13 > 0\;\forall \alpha \in (0,\frac{1}{2})$.

We next show that this positivity is not preserved under convexification. Namely, we will show $\gamma_{\operatorname{co}(\mathcal F_2),\alpha} = 0$. Fix any $\pi \in \Pi$. Choose two functions $f_1,f_2 \in \mathcal F_2$ such that $f_1^{-1}(\{1\}) \cap f_2^{-1}(\{1\}) = \{\pi\}$
and $f_1^{-1}(\{\frac{1}{2}\}) = f_2^{-1}(\{\frac{1}{2}\})$.
Define $\tilde f := \tfrac12(f_1 + f_2) \in \operatorname{co}(\mathcal F_2)$. Then we have $\tilde f(\pi)=1$ and for all $\bar\pi \neq \pi$, we have $\tilde f(\bar\pi)=\tfrac12$. Therefore, for any distribution $p$, there exists some arm $\tilde{\pi}$ with probability mass $0$ from $p$ and there is some function $\tilde{f}$ such that $\tilde{f}(\tilde{\pi})=1$ and every $\pi \neq \tilde{\pi}$ has $\tilde{f}(\pi) = \frac{1}{2}$. Thus $\mathbb P_{\pi \sim p}\bigl(\sup_{\pi^*} \tilde{f}(\pi^*)- \tilde f(\pi) \leq  \alpha\bigr)=0$,
$\forall \alpha \in (0,\frac{1}{2})$. This implies $\gamma_{\operatorname{co}(\mathcal{F}_2),\alpha}=0\;\forall \alpha \in (0,\frac{1}{2})$.


Finally, we show there is no finite $(\alpha,\beta)$-distribution cover for $\mathcal{F}_2\;\forall \alpha \in (0,\frac{1}{2}),\forall \beta \in (0,\frac{1}{3})$. Together with $\gamma_{\mathcal{F}_2,\alpha}>0$ and $\gamma_{\operatorname{co}(\mathcal{F}_2),\alpha}=0$, this shows that convexity is essential for Lemma~\ref{conjecture1}. 

For the purpose of analysis, we consider auxiliary function class $\mathcal F_3 := \bigl\{ {f:\Pi\to\{-1,0,1\}} \;\bigm|\; {\lambda\bigl(\{f(\pi)=1\}\bigr)}={\lambda\bigl(\{f(\pi)=-1\}\bigr)}=\frac13 \bigr\}$. Suppose, for contradiction, that there exists a finite distribution cover $\mathcal S = \{p_1,\dots,p_N\} \subset \Delta(\Pi)$. Partition $[0,1]$ into $3^M$ equal subintervals $A_1,\dots,A_{3^M}$, where $M$ is sufficiently large such that $N3^{2M/3}\leq3^{M-1}$ and $\log N \cdot 3^{-\frac{M}{3}}\leq \frac{1}{100}$. We first perform pruning step: remove all indices $i$ such that $\max_{k} p_k(A_i) > 3^{-2M/3}$. The number of removed intervals is bounded by
$$\#\Bigl\{i:\max_k p_k(A_i)>3^{-2M/3}\Bigr\}
\le \sum_{k=1}^N \#\left\{i:p_k(A_i)>3^{-2M/3}\right\}
\le N3^{2M/3} \leq 3^{M-1}$$
Next, we choose any $2\cdot3^{M-1}$ intervals from the remaining intervals and denote this set of indices $i$ by $R$, and let \(B := [3^M]\setminus R\). Then we do random construction: let $I$ be uniformly random (without replacement) half of the set $R$, and define
\[
f(\pi)
=
\begin{cases}
1, & \pi\in \bigcup_{i\in I}A_i,\\
-1, & \pi\in \bigcup_{i\in R\setminus I}A_i,\\
0, & \pi\in \bigcup_{i\in B}A_i.
\end{cases}
\] 
We can see $f \in \mathcal{F}_3$.  Then focus on the set $R$:
    \begin{equation*}
        \begin{aligned}
        \mathbb{E}\!\left[\max_k \int_R f d p_k\right] 
            &=\mathbb{E}\!\left[\max_k \sum_{i\leq 3^M, i \notin B}p_k(A_i)(2\mathbb{I}[i \in I]-1)\right] \\
            & \leq 2\sqrt{\log N} \max_k \sqrt{\left(\sum_{i\leq 3^M, i\notin B}p_k(A_i)^2 \right)} 
            \\ &\leq 2\sqrt{\log N}  \sqrt{3^M 3^{-\frac{4}{3}M}} 
            \leq 2\sqrt{\log N}  \sqrt{3^{-\frac{M}{3}}} 
            \leq \frac{1}{5}.
        \end{aligned}
    \end{equation*}
    The first inequality is due to Massart's Lemma (Lemma \ref{massart_lemma}). The last inequality is as $\log N\cdot 3^{-\frac{M}{3}}\leq \frac{1}{100}$. Therefore, there exists a function $f \in \mathcal{F}_3$ such that $\max_k \int_R f dp_k \leq \frac{1}{5}$. 

Last, we rescale to $\{0,\frac{1}{2},1\}$-valued function: we choose $f'(x)=\frac{f(x)+1}{2} \in \{0,\frac{1}{2},1\}$, and note that $f' \in \mathcal{F}_2$. For every $\alpha \in (0,\frac{1}{2})$,  $$p_k\left(\sup_{\pi^*}f'(\pi^*) - f'(\pi) \le \alpha\right)=p_k(f'=1)=\int_R f'dp_k=\frac{1}{2}p_k(R)+\frac{1}{2}\int_R f d p_k\leq\frac{1}{2}\cdot1+\frac{1}{10}=\frac{3}{5}.$$ Therefore, for $\alpha<\frac{1}{2},\beta<\frac{1}{3}$, this $f'$ is not covered by $\mathcal{S}$.  Since this argument holds for any possible finite $\mathcal{S}$, we conclude there is no finite $(\alpha,\beta)$-distribution cover.

\end{example}




\section{Conclusion and Open Problems}
\label{conclusion}
In this work, we study the learnability of the adversarial noisy bandit problem and establish a complete characterization of learnability in terms of a new complexity measure, the \emph{convexified generalized maximin volume}. Interestingly, our results show that learnability against oblivious adversaries and learnability against adaptive adversaries coincide in the noisy bandit setting. Along the way, we introduced two additional complexity measures, the \emph{hitting set} and the \emph{distribution covering number}. These notions may be of independent interest and could provide useful intuition for the study of other learning problems.

We conclude with two important open problems:

\begin{itemize}

    \item The first open question concerns the quantitative relationship between $\gamma_{\operatorname{co}(\mathcal{F}),\alpha}$ and $(\alpha,\beta)$-distribution covers. 
    In particular, can one establish a tight connection between $\gamma_{\operatorname{co}(\mathcal{F}),\alpha}$ and the minimal size of $(\alpha,\beta)$-distribution covers? 
    \item Our second open question is to understand the spectrum of optimal regret in terms of the \emph{convexified generalized maximin volume} or the \emph{distribution covering number}. More specifically, can we construct, for every possible regret rate, a function class whose optimal regret attains that rate?
\end{itemize}

\newpage



\bibliography{ref}

\appendix



\section{Auxiliary Lemma}

\begin{lemma}[Massart's Lemma]
\label{massart_lemma}
Let $\mathcal{A} \subseteq \mathbb{R}^m$ be a finite set, with $r = \max_{\mathbf{a} \in \mathcal{A}} \|\mathbf{a}\|_2$. Suppose \(m\) is even, and let
\(\sigma=(\sigma_1,\ldots,\sigma_m)\) be uniformly distributed over $\left\{\sigma\in\{-1,1\}^m:\sum_{i=1}^m \sigma_i=0\right\}$.
Then
$$
\mathbb{E}
\left[
\sup_{\mathbf{a} \in \mathcal{A}}
\sum_{i=1}^m \sigma_i a_i
\right]
\leq
2r \sqrt{\log |\mathcal{A}|},$$
where $a_i$ denotes the $i$-th component of the vector $\mathbf{a}$.
\end{lemma}

\begin{lemma}[Markov's Inequality]
\label{markov_inequality}
Let $X$ be a nonnegative random variable. Then, for any $a>0$,
$$\mathbb{P}(X \geq a) \leq \frac{\mathbb{E}[X]}{a}.$$
\end{lemma}

\begin{lemma}[Chernoff bound under conditional domination]
\label{martingale_chernoff_bound}
Let \(X_1,\dots,X_k\) be \(\{0,1\}\)-valued random variables adapted to
a filtration \((\mathcal F_j)_{j=0}^k\). Suppose that for some \(\mu\in(0,1)\), $\mathbb E[X_j\mid \mathcal F_{j-1}] \le \mu
\;\forall j$
Then for every \(\theta\in(\mu,1)\),
\[
\mathbb{P}\left(\sum_{j=1}^k X_j\ge \theta k\right)
\le
\exp\left(-D(\theta\|\mu)k\right),
\]
where $D(\theta\|\mu)=\theta\log\frac{\theta}{\mu}+(1-\theta)\log\frac{1-\theta}{1-\mu}$ denotes the binary relative entropy.
\end{lemma}
\begin{proof}
For every \(\lambda > 0\), $\mathbb E \exp\left(\lambda \sum_{i}X_i\right)
\le
(1-\mu+\mu e^\lambda)^{k}$ by iterating the conditional expectation bound. Markov's inequality gives $\mathbb{P}\left(
\sum_{i}X_i \ge \theta k
\right)
\le
\exp(-\lambda \theta k)
(1-\mu+\mu e^\lambda)^{k}$. Optimizing over \(\lambda>0\) yields the desired inequality.
\end{proof}
\section{Exp3-IX Algorithm}
\label{exp3_section}


    



\begin{algorithm2e}[htbp]
    \caption{Exp3-IX}
    \label{alg:exp3-ix}

    \DontPrintSemicolon
    \SetAlgoNoLine
    \SetAlgoNoEnd
    \LinesNotNumbered
    \SetNlSty{textnormal}{}{.}

    \SetKwInput{KwParameters}{Parameters}
    \SetKwInput{KwInitialization}{Initialization}
    \SetKwFor{ForRound}{for}{repeat}{end}

    \KwParameters{$\eta > 0$, $\gamma > 0$.}
    \KwInitialization{$w_{1,i}=1$ for all $i\in[K]$.}

    \ForRound{$t=1,2,\ldots,T$,}{
        \nl
        $\displaystyle
        p_{t,i}
        =
        \frac{w_{t,i}}{\sum_{j=1}^{K} w_{t,j}}
        $\;

        \nl
        Draw $I_t \sim \boldsymbol{p}_t
        = (p_{t,1},\ldots,p_{t,K})$\;

        \nl
        Observe loss $\ell_{t,I_t}=1-r_t$\;
        \nl
        $\displaystyle
        \widetilde{\ell}_{t,i}
        \leftarrow
        \frac{\ell_{t,i}}{p_{t,i}+\gamma}
        \mathbb{I}\{I_t=i\}
        $
        for all $i\in[K]$\;

        \nl
        $\displaystyle
        w_{t+1,i}
        \leftarrow
        w_{t,i}\,e^{-\eta\widetilde{\ell}_{t,i}}
        $
        for all $i\in[K]$\;
    }
\end{algorithm2e}


\begin{lemma}[Corollary of \cite{neu2015explore} Theorem 1]
\label{exp3_guarantee}
    For a sequence of mean-reward vectors \(\mu_t\in[0,1]^K\) chosen by an adaptive adversary, with $K\geq 2$ and $r_t \in [0,1]$ satisfying  $\mathbb E[r_t\mid \mathscr H_{t-1},I_t]=\mu_{t,I_t}$, where $\mathscr{H}_{t-1}$ denotes the information from the interaction before round \(t\). With $\eta_t=2\gamma_t=\sqrt{\frac{2\log K}{KT}}$, Algorithm \ref{alg:exp3-ix} satisfies  
\[
\mathbb E\!\left[
\max_{i\le K}\sum_{t=1}^T\mu_{t,i}
-\sum_{t=1}^T\mu_{t,I_t}
\right]
\le 2\sqrt{2KT\log K}
\]
\end{lemma}

\end{document}